\documentclass[conference]{IEEEtran}
\usepackage{times}
\usepackage{subfigure}
\usepackage{graphicx}
\usepackage{amsmath,amssymb,amsthm}
\usepackage{mathtools}
\usepackage{xcolor}
\usepackage[ruled,vlined]{algorithm2e}

\newtheorem{proposition}{Proposition}
\newtheorem{remark}{Remark}

\usepackage[numbers]{natbib}
\usepackage{multicol}
\usepackage{url}

\newcommand{\subparagraph}{}
\usepackage{titlesec}
\titlespacing*{\subsection}{0pt}{0.15\baselineskip}{0.05\baselineskip}
\titlespacing*{\section}{0pt}{0.6\baselineskip}{0.5\baselineskip}

\newcommand{\argmax}{\operatornamewithlimits{arg\,max}}

\begin{document}
\abovedisplayskip=5pt
\abovedisplayshortskip=5pt
\belowdisplayskip=5pt
\belowdisplayshortskip=5pt

\title{Reinforcement Learning based Control of\\Imitative Policies for Near-Accident Driving\vspace{-9px}}

\author{
\authorblockN{Zhangjie Cao$^*$$^1$, Erdem B\i y\i k$^*$$^2$, Woodrow Z. Wang$^1$,\\ Allan Raventos$^3$, Adrien Gaidon$^3$, Guy Rosman$^3$, Dorsa Sadigh$^{1,2}$}
\authorblockA{$^1$Computer Science, Stanford University, $^2$Electrical Engineering, Stanford University, $^3$Toyota Research Institute}
\authorblockA{Emails: \{caozj18, ebiyik, wwang153, dorsa\}@stanford.edu, \{allan.raventos, adrien.gaidon, guy.rosman\}@tri.global}
\authorblockA{$^*$ First two authors contributed equally to this work.\vspace{-12px}}}


%

\maketitle

\begin{abstract}
Autonomous driving has achieved significant progress in recent years, but autonomous cars are still unable to tackle high-risk situations where a potential accident is likely. In such near-accident scenarios, even a minor change in the vehicle's actions may result in drastically different consequences. To avoid unsafe actions in near-accident scenarios, we need to fully explore the environment. However, reinforcement learning (RL) and imitation learning (IL), two widely-used policy learning methods, cannot model rapid phase transitions and are not scalable to fully cover all the states. To address driving in near-accident scenarios, we propose a hierarchical reinforcement and imitation learning (\textsc{H-ReIL}) approach that consists of low-level policies learned by IL for discrete driving modes, and a high-level policy learned by RL that switches between different driving modes. Our approach exploits the advantages of both IL and RL by integrating them into a unified learning framework. Experimental results and user studies suggest our approach can achieve higher efficiency and safety compared to other methods. Analyses of the policies demonstrate our high-level policy appropriately switches between different low-level policies in near-accident driving situations.
\end{abstract}

\IEEEpeerreviewmaketitle

\section{Introduction}
Recent advances in learning models of human driving behavior have played a pivotal role in the development of autonomous vehicles. Although several milestones have been achieved (see \cite{pomerleau1989alvinn,bojarski2016end,amini2019variational,sadigh2016planning,sadigh2016information,sadigh2018planning,lee2017desire,biyik2018batch,codevilla2019exploring,kwon2020when,basu2019active,wulfmeier2017large} and references therein), the current autonomous vehicles still cannot make safe and efficient decisions when placed in a scenario where there can be a high risk of an accident
(a near-accident scenario). For example, an autonomous vehicle needs to be able to coordinate with other cars on narrow roads, make unprotected left turns in busy intersections, yield to other cars in roundabouts, and merge into a highway in a short amount of time. The left panel of Fig. \ref{fig:arch} shows a typical near-accident scenario: The ego car (red) wants to make an unprotected left turn, but the red truck occludes the oncoming blue car, making the ego car fail to notice the blue car, which can potentially result in a collision. 
Clearly, making suboptimal decisions in such near-accident scenarios can be dangerous and costly, and is a limiting factor on the road to safe wide-scale deployment of autonomous vehicles.

One major challenge when planning for autonomous vehicles in near-accident scenarios is the presence of \emph{phase transitions} in the car's policy.
Phase transitions in autonomous driving occur when small changes in the critical states -- the ones we see in near-accident scenarios -- require dramatically different actions of the autonomous car to stay safe. 
For example, the speed of the blue car in Fig.~\ref{fig:arch} can determine the ego car's policy: if it slows down, the ego car can proceed forward and make the left turn; however, a small increase in its speed would require the ego car to stop and yield. The rapid phase transition requires a policy that can handle such non-smooth transitions. Due to the non-smooth value function, an action taken in one state may not generalize to nearby states. Hence, when training a policy, our algorithms must be able to visit and handle all the critical states individually, which can be computationally inefficient. 




Reinforcement learning (RL) \cite{wulfmeier2017large,makantasis2019deep,sallab2017deep} and imitation learning (IL) \cite{pomerleau1989alvinn,bojarski2016end,amini2019variational,codevilla2018end,zhang2017query,bansal2018chauffeurnet,chen2015deepdriving,bojarski2017explaining,kuefler2017imitating,pan2017virtual,huang2019uncertainty,muller2018driving,paden2016survey} are two promising learning-based approaches for autonomous driving. RL explores the state-action space to find a policy that maximizes the reward signals while IL imitates the behavior of the agent from expert demonstrations. However, the presence of rapid phase transitions makes it hard for RL and IL to capture the policy because they learn a smooth policy across states. Furthermore, to achieve full coverage, RL needs to explore the full environment while IL requires a large amount of expert demonstrations covering all states. Both are prohibitive since the state-action space in driving is continuous and extremely large.

In this paper, \emph{our key insight is to model phase transitions as optimal switches, learned by reinforcement learning, between different modes of driving styles, each learned through imitation learning.}
In real world driving, various factors influence the behaviors of human drivers, such as efficiency (time to destination), safety (collision avoidance), etc. Different modes characterize different trade-offs of all factors. For example, the aggressive mode cares more about efficiency so it always drives fast in order to reach the destination in minimal time. The timid mode cares more about safety, so it usually drives at a mild speed and pays attention to all potential threats. Switching from one mode to another can model the rapid phase transition conditioned on the environment changes. 

Using these modes, we propose a new algorithm \textbf{Hierarchical Reinforcement and Imitation Learning} (\textsc{H-ReIL}), which is composed of a high-level policy learned with RL that switches between different modes and low-level policies learned with IL, each of which represents a different mode. 

\begin{figure*}
    \centering
    \subfigure{
\includegraphics[width=.24\textwidth]{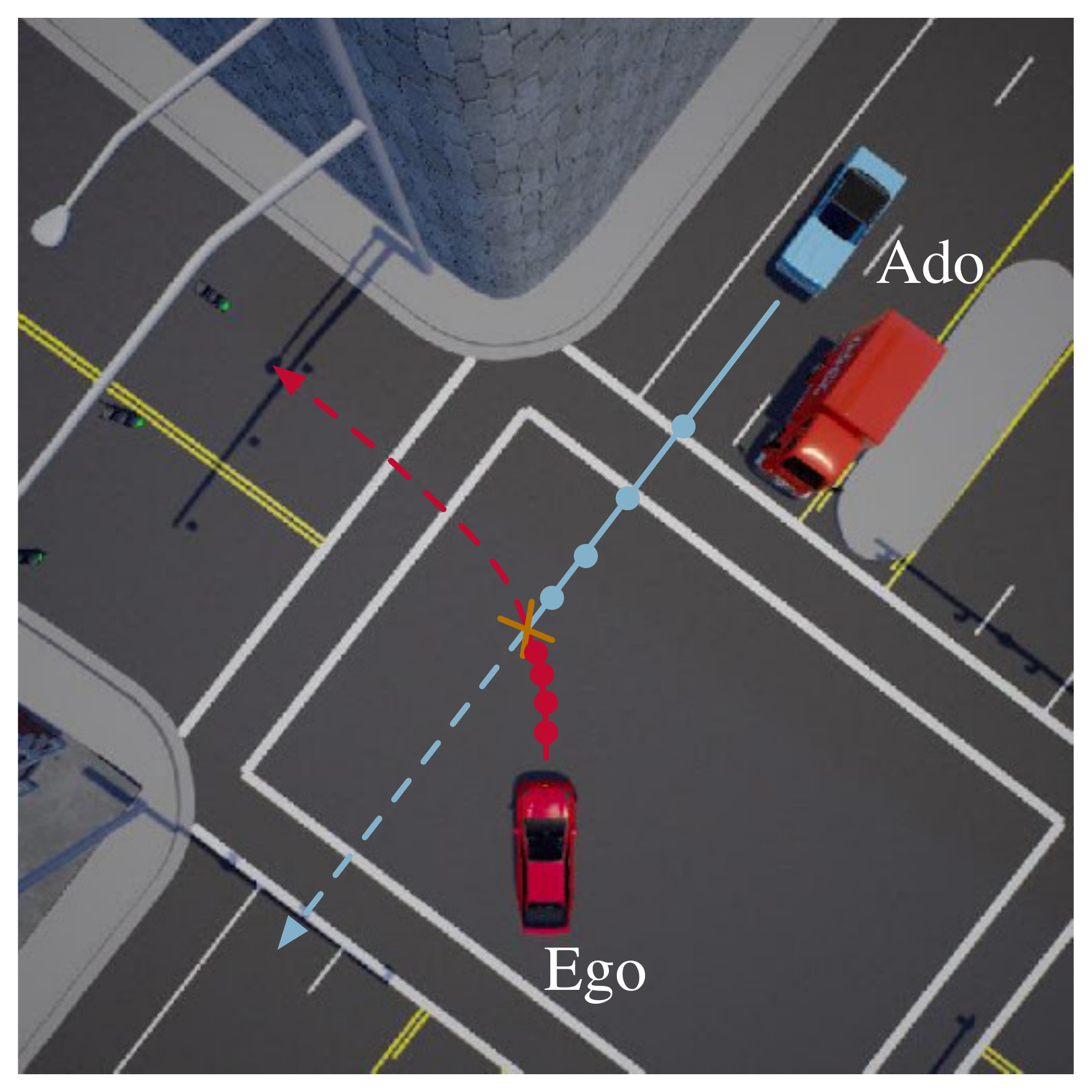}}
    \subfigure{
\includegraphics[width=.74\textwidth]{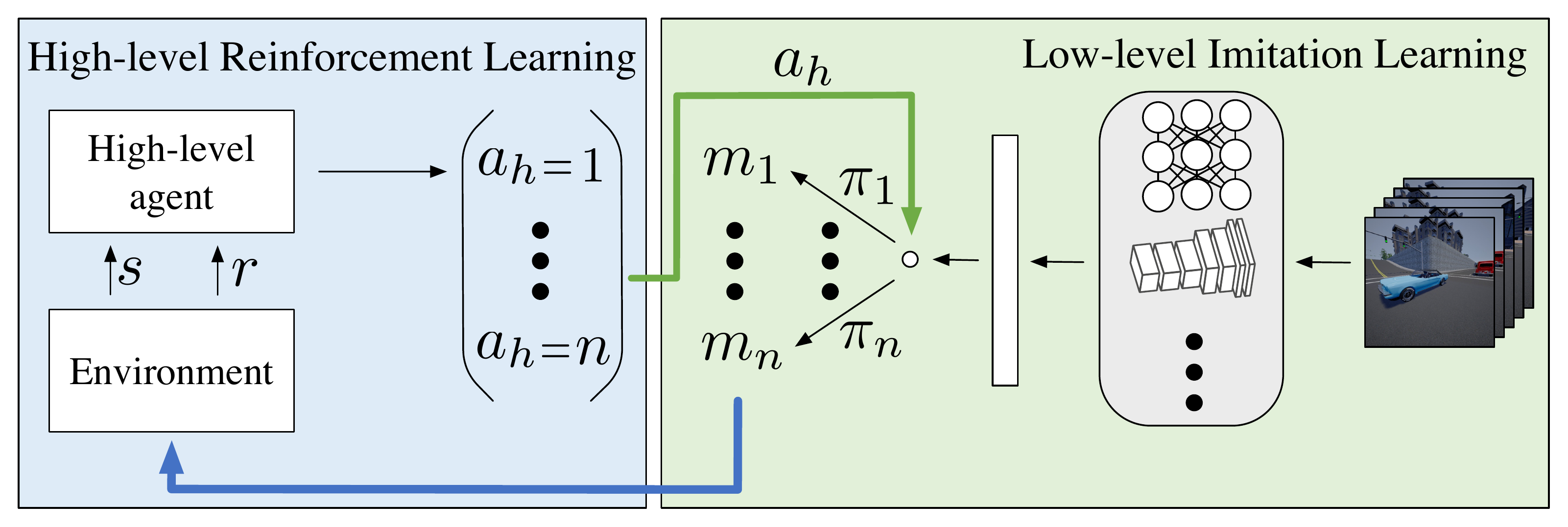}}
\vspace{-20px}
\caption{The left part of the figure is a typical near-accident scenario: The ego car (red car) turns left but the truck occludes the blue car, which causes the ego car to overlook the blue car and collide with it at time step 5. The right part of the figure is the overall architecture of the proposed hierarchical reinforcement learning and imitation learning model. The right green square shows the low-level imitation learning part, where the low-level policies are learned by the conditional imitation framework. All the policies share the same feature extractor and split to different branches in later layers for action prediction, where each corresponds to one mode. The branch is selected by external input $a_h$ from high-level reinforcement learning. The low-level policies are learned from expert demonstrations by imitation learning.  The left blue square shows the high-level reinforcement learning part, where the high-level agent interacts with the environment to learn the high-level policy, which selects the best low-level policy branch through the high-level action $a_h$ at different states.}\label{fig:arch}
\vspace{-15px}
\end{figure*}

Using our proposed approach, the low-level policy for each mode can be efficiently learned with IL even with only a few expert demonstrations, since IL is now learning a much simpler and specific policy by sticking to one driving style with little phase transition. We emphasize that RL would not be a reasonable fit for learning the low-level policies as it is difficult to define the reward function. For example, designing a reward function for the aggressive mode that exactly matches an aggressive human driver's behavior is non-trivial.

For the high-level policy, RL is a better fit since we need to learn to maximize the return based on a reward that contains a trade-off between various terms, such as efficiency and safety. Furthermore, the action space is now reduced from a continuous space to a finite discrete space. IL does not fit to the high-level policy, because it is not natural for human drivers to accurately demonstrate how to switch driving modes.

We therefore combine RL at the high-level and IL at the low-level in our proposed hierarchical model, which can utilize both approaches and learn driving policies in a wide variety of settings, including near-accident driving scenarios.

Our main contributions in this paper are three-fold:
\begin{itemize}
    \item We develop a \textbf{Hierarchical Reinforcement and Imitation Learning (\textsc{H-ReIL})} approach composed of a high-level policy learned with RL, which switches optimally between different modes, and low-level policies learned with IL, which represent driving in different modes. 
    
    \item We demonstrate and assess our proposed \textsc{H-ReIL} model on two different driving simulators in a set of near-accident driving scenarios. Our simulations demonstrate that the learned hierarchical policy outperforms imitation learning policies, the individual policies learned for each mode, and a policy based on random mixtures of modes, in terms of efficiency and safety.
    
    \item We finally conduct a user study in which human subjects compare trajectories generated by \textsc{H-ReIL} and the compared methods to demonstrate \textsc{H-ReIL}'s ability to generate safe and efficient policies. The results show the users significantly prefer the \textsc{H-ReIL} driving policies compared to other methods in near-accident scenarios.
\end{itemize}



\section{Related Work}
\noindent \textbf{Rule-based Methods.} Traditional autonomous driving techniques are mostly based on manually designed rules \cite{schwarting2018planning,urmson2008autonomous,montemerlo2008junior}. However, it is tedious, if not impossible, to enumerate all driving rules and norms to deal with all the states. Therefore, rule-based methods often cause the vehicle to drive in an unnatural manner or completely fail in unexpected edge cases.

\noindent \textbf{Imitation Learning (IL).} ALVINN was one of the first instances of IL applied to driving \cite{pomerleau1989alvinn}. Following ALVINN, \citet{muller2006off} solved off-road obstacle avoidance using behavior cloning. IL learns driving policies on datasets consisting of off-policy state-action pairs. However, they suffer from potential generalization problems to new test domains due to the distribution shift. \citet{ross2011no} address this shortcoming by iteratively extending the base dataset with on-policy state-action pairs, while still training the base policy offline with the updated dataset. \citet{bansal2018chauffeurnet} augment expert demonstrations with perturbations and train the IL policy with an additional loss penalizing undesired behavior. Generative Adversarial Imitation Learning \cite{ho2016generative,song2018multi} proposes to match the state-action occupancy between trajectories of the learned policy and the expert demonstrations.

A major shortcoming of IL is that it requires a tremendous amount of expert demonstrations. Conditional imitation learning (CoIL) \cite{codevilla2018end} extends IL with high-level commands and learns a separate IL model for each command. Although it improves data-efficiency, high-level commands are required at test time, e.g., the direction at an intersection. In our setting, each high-level command corresponds to a different driving mode. Instead of depending on drivers to provide commands, we would like to learn the optimal mode-switching policy.


\noindent \textbf{Inverse Reinforcement Learning (IRL).} Originally proposed to address the learning problem in a Markov decision process (MDP) without an explicitly given reward function \cite{abbeel2004apprenticeship}, IRL aims to recover the reward function from expert demonstrations. The reward is typically represented by a weighted sum of several reward features relevant to the task. IRL learns those weights by observing how experts perform the task. \citet{abbeel2004apprenticeship} tune the weights to match the expected return of the expert trajectories and the optimal policy. \citet{ziebart2008maximum} further add a maximum entropy regularization. Following \cite{levine2012continuous}, \citet{finn2016guided} improve the optimization in \cite{ziebart2008maximum}.

Similar to IL, IRL also suffers from the requirement of a large amount of expert demonstrations. It is also difficult and tedious to define reward features that accurately characterize efficiency and safety in all scenarios. Thus, IRL is not fit for learning driving policies in near-accident scenarios.

\noindent \textbf{Reinforcement Learning (RL).} RL has been applied to learn autonomous driving policies \cite{sallab2017deep,chen2019model,youssef2019deep,shalev2016safe}. RL explores the environment to seek the action that maximizes the expected return for each state based on a pre-defined reward function. However, it suffers from the fact that the state-action space for driving is extremely large, which makes it very inefficient to explore. \citet{chen2019model} try to alleviate this problem by augmenting RL with Model Predictive Control (MPC) to optimally control a system while satisfying a set of constraints. \citet{tram2019learning} combine RL with MPC to shrink the action space, however the MPC is based on driving rules, which are difficult to exhaustively define and enumerate. Finally, \citet{gupta2019relay} proposed using RL to fine-tune IL policies for long-horizon, multi-stage tasks, different than our problem setting.

\noindent \textbf{Hierarchical Reinforcement Learning.} Hierarchical RL is motivated by feudal reinforcement learning \cite{dayan1993feudal}, which first proposes a hierarchical structure for RL composing of multiple layers: the higher layer acts as a manager to set a goal for the lower layer, which acts as a worker to satisfy the goal. Hierarchical RL enables efficient exploration for the higher level with a reduced action space, i.e. goal space, while making RL in the lower level easier with an explicit and short-horizon goal. Recent works extended hierarchical RL to solve complex tasks \cite{kulkarni2016hierarchical,vezhnevets2017feudal,stulp2011hierarchical,strudel2019combining,wu2020model}. \citet{le2018hierarchical} proposed a variant of hierarchical RL, which employs IL to learn the high-level policy to leverage expert feedback to explore the goal space more efficiently. Recently, more related to our work, \citet{qureshi2020composing} proposed using deep RL to obtain a mixture between task-agnostic policies. However in our case, low-level policies are not task-agnostic and are produced by IL on the same tasks, so it is arguably sufficient to discretely switch between them. Finally, \citet{nair2018overcoming} use expert demonstrations to guide the exploration of RL. 

However, for near-accident scenarios, most off-the-shelf hierarchical RL techniques do not address the problem of driving safely and efficiently, because it is difficult to define the reward function for low-level RL. We instead construct a hierarchy of RL and IL, where IL is in the low-level to learn a basic policy for each mode and RL is in the high-level, similar to \cite{comanici2010optimal}, to learn a mode switching policy that maximizes the return based on a simpler pre-defined reward function.




\section{Model}

\subsection{Problem Setting}
We model traffic as a partially observable Markov decision process (POMDP): $P_l = \langle \mathcal{S}, \Omega, O, \mathcal{A}, f, R \rangle$ where the agent is the ego car. The scenario terminates either by a collision, by reaching the destination, or by a time-out, which forces the POMDP to be finite horizon. $\mathcal{S}$ is the set of states, $\Omega$ is the set of observations, $O$ is the set of conditional observation probabilities, $\mathcal{A}$ is the set of actions, and $f$ is the transition function. Each state $s^t\in \mathcal{S}$ consists of the positions and velocities of all the vehicles at time step $t$. Each action $a^t \in \mathcal{A}$ is the throttle and the steering control available for the ego car. At each time step $t$, all vehicles move and the state $s^t$ is transitioned to a new state $s^{t+1}$ according to $f$, which we model as a probability distribution, $P(s^{t+1} | s^t,a^t) = f(s^t,a^t,s^{t+1})$, where the stochasticity comes from noise and the uncertainty about the other vehicles' actions. The agent receives an observation $o^t\in \Omega$ with a likelihood conditioned on the state $s^t$, \textit{i.e.} $O(o^t|s^t)$. For example, if some vehicles are occluded behind a building, their information is missing in the observation. Finally, the agent receives a reward $R(s^t,a^t)$ at each time step $t$, which encodes desirable driving behavior.


\subsection{\textsc{H-ReIL} Framework}
We design the \textsc{H-ReIL} framework using a set of $n$ experts, each representing its own mode of driving. Following different modes, the experts are not necessarily optimal with respect to the true reward function. For example, the modes can be aggressive or timid driving. We denote the corresponding policies by $\pi_1,\dots,\pi_n$; where $\pi_i:\Omega\to\mathcal{A}$, $\forall i$. Our goal is to learn a policy $\Pi$ that switches between the modes to outperform all $\pi_i$ in terms of cumulative reward.

As shown in the right panel of Fig.~\ref{fig:arch}, we divide the problem into two levels where $\pi_i\mid_{i=1}^{n}$ are low-level policies learned with IL using the data coming from experts, and the high-level agent learns $\Pi$ with RL using a simulator of the POMDP. 


\noindent\textbf{Low-Level Imitation Learning Policy.} Unlike \cite{kulkarni2016hierarchical} and \cite{le2018hierarchical}, which employ RL in the low-level of the hierarchy, we employ IL to learn low-level policies $\pi_i$, because each low-level policy sticks to one driving style, which behaves relatively consistently across states and requires little rapid phase transitions. Hence, the actions taken in nearby states can generalize to each other easily. Therefore, the simpler policy can be learned by IL easily with only a few expert demonstrations $\mathcal{H}_i=\{o_i^t,a_i^t\}|_{t=1}^{K}$, consisting of observation-action pairs for each mode $m_i$. Here we use Conditional Imitation Learning (CoIL) \cite{codevilla2018end} as our IL method.
We define the loss as
\begin{equation}\label{eqn:IL}
    l_{IL} = \frac{1}{n}\sum_{i=1}^n\frac{1}{K}\sum_{t=1}^{K}\ell_1 (a_i^t, {\pi}_i(o_i^t)),
\end{equation}
where we take the mean over $L1$ distances. As in CoIL, we model ${\pi}_i\mid_{i=1}^n$ using a neural network with branching at the end. Each branch corresponds to an individual policy ${\pi}_i$. We present the details of the networks in Section~\ref{sec:implementation_details}.

\noindent \textbf{High-Level Reinforcement Learning Policy.} After training the low-level policies, we build the high-level part of the hierarchy: We train a high-level policy $\Pi$ to select which of the policies from $S_{\pi} = \{\pi_i\}_{i=1}^n$ the ego car should follow. This high-level decision is made every $t_s$ time steps of $P_l$.

We model this high-level problem as a new POMDP, called $P^{t_s}_h$, where the states and observations are the same as the original POMDP $P_l$, but the actions choose which driving mode to follow. For example, if the action is $2$, then the ego car follows $\pi_2$ for the next $t_s$ time steps in $P_l$, which is a single time step in $P^{t_s}_h$.
Formally, $P^{t_s}_h=\langle \mathcal{S}, \Omega, \mathcal{O}, \mathcal{A}_{h}, f^{t_s}_h, R^{t_s}_h \rangle$ and the new action space $\mathcal{A}_{h}$ is a discrete space, $\{1,2,...,n\}$, representing the selection of low-level policies. The new transition function $f^{t_s}_h(s^t,a_h,s^{t+1})$ gives the probability of reaching $s^{t+1}$ from $s^t$ by applying policy $\pi_{a_h}$ for $t_s$ consecutive time steps in $P_l$. Similarly, the new reward function accumulates the reward from $P_l$ over $t_s$ time steps in which the policy $\pi_{a_h}$ is followed.

Then, our goal in this high-level hierarchy is to solve:
\begin{align}
    \argmax_{\Pi}\; & \mathbb{E}\left[\sum_{j} \sum_{o^j} O(o^j\mid s^j)R_h^{t_s}(s^j,\Pi(o^j))\right] \nonumber\\
    \textrm{subject to } & s^{j+1} \sim f_h^{t_s}(s^j,\Pi(o^j),s^{j+1})\textrm{ for }\forall j
    \label{eq:rl_optimization}
\end{align}
where we use indexing $j$ to denote the time steps of $P^{t_s}_h$. As shown in Fig. \ref{fig:arch}, we attempt to solve \eqref{eq:rl_optimization} using RL. In $P_h^{t_s}$, the action space is reduced from continuous to discrete, which eases the efficient exploration of the environment. Furthermore, it is now much easier to define a reward function because the ego car already satisfies some properties by following the policies learned from expert demonstrations. For example, with a high enough $t_s$, we do not need to worry about jerk, because the experts naturally give low-jerk demonstrations. Therefore, we design a simple reward function consisting of the efficiency and safety terms ($R_e$ and $R_s$). $R_e$ is negative in every time step, so that the agent will try to reach its destination as quickly as possible. $R_s$ gets an extremely large negative value if a collision occurs. Otherwise, it is $0$.

Besides, setting $t_s\!>\!1$ reduces the number of time steps in an episode and makes the collision penalty, which appears at most once per episode, less sparse. With the new action space, transitions, and reward function, we can train the high-level policy with any RL algorithm (PPO \cite{schulman2017proximal} in this paper). Algorithm~\ref{algorithm} outlines our training algorithm.

\begin{algorithm}[h]
\SetAlgoLined
\KwIn{Expert demonstrations $\mathcal{H}_1,...,\mathcal{H}_n$, POMDP $P^{t_s}_h=\langle \mathcal{S}, \Omega, \mathcal{O}, \mathcal{A}_{h}, f^{t_s}_h, R^{t_s}_h \rangle$}
\KwOut{Low-level policies $\pi_i\mid_{i=1}^N$, high-level policy $\Pi$}
Train low-level policies $\pi_i\mid_{i=1}^N$ with demonstrations $\mathcal{H}_i\mid_{i=1}^N$ to minimize the loss in Eqn.~\eqref{eqn:IL}.\\
Train high-level policy $\Pi$ using $\pi_i\mid_{i=1}^N$ and $P^{t_s}_h$ according to \eqref{eq:rl_optimization} with PPO.\\
\Return $\pi_i\mid_{i=1}^N$ and $\Pi$
\caption{H-ReIL Training Algorithm}\label{algorithm}
\end{algorithm}
\vspace{-7px}

\subsection{Analysis of H-ReIL}

\begin{proposition} \label{prop:random}
Let's consider a POMDP with a fixed finite horizon $T$, for which we have $n$ low-level policies. Let's call the expected cumulative reward for the optimal and worst high-level control sequences $U^*$ and $U'$, respectively. If there exists a scalar $a>U^*-U'$ such that the expected cumulative rewards of keeping the same low-level policy are smaller than $U^* - a + a/n^T$;
then there exists a probability distribution $p$ such that randomly switching the policies with respect to $p$ is better than keeping any of the $n$ low-level policies.
\end{proposition}
\begin{proof}
Let $p$ be a uniform distribution among the low-level policies. Then, each possible control sequence has a $1/n^T$ probability of being realized. This guarantees that the expected cumulative reward of this random policy is larger than: $\frac{n^T-1}{n^T}(U^*-a) + \frac{1}{n^T}(U^*) = U^* - a + a/n^T$.

While this is a worst-case bound, it can be shown that the expected cumulative reward of a random policy can be higher if the optimal high-level control sequence is known to be imbalanced between the modes. In that case, a better lower bound for random switching is obtained by a $p$ maximizing the probability of the optimal sequence being realized.
\end{proof}

For a different interpretation of \textsc{H-ReIL}, one can think of the true driving reward $R$ as a sum of $n$ different terms; such as, for $n=2$, $R(s^t,a^t)=R_e(s^t,a^t) + R_s(s^t,a^t)$ where $R_e$ denotes the part of the reward that is more associated with efficiency, and $R_s$ with safety. Then, strictly aggressive drivers optimize for $\alpha R_e(s^t,a^t) + (2-\alpha)R_s(s^t,a^t)$ for some $1<\alpha\leq 2$, whereas strictly timid drivers try to optimize the same reward with $0\leq \alpha <1$. One may then be tempted to think there exists a high-level stationary random switching distribution $p$ that outperforms both the aggressive and timid drivers, because the true reward function is in the convex hull of the individuals' reward functions for each $s^t\in\mathcal{S},a^t\in\mathcal{A}$. However, even with this reward structure and hierarchy, the existence of such a $p$ is not guaranteed without the assumptions of Proposition~\ref{prop:random} (or other assumptions).

\begin{remark}\label{remark:mdp}
With the reward structure that can be factorized such that each mode weighs some terms more than the others and the true reward is always in the convex hull of them, there may not exist a high-level stationary random-switching strategy that outperforms keeping a single low-level policy. 
\end{remark}
\begin{proof}
Consider the $4$-state deterministic MDP with a finite-horizon $T$ shown in Fig.~\ref{fig:remark_mdp}. There are only two actions, represented by solid ($a\!=\!1$) and dashed ($a\!=\!2$) lines. The rewards for each state-action pair are given in a tuple form $r\!=\!(R_e,R_s)$ where the true reward is $R(s^t,a^t)=R_e(s^t,a^t) + R_s(s^t,a^t)$. Consider two modes optimizing $\alpha R_e(s^t,a^t) + (2-\alpha) R_s(s^t,a^t)$, one for $\alpha\!=\!1.8$ and the other for $\alpha\!=\!0.2$. While the former will always take $a\!=\!1$, the latter will keep $a\!=\!2$. Both policies will achieve a true cumulative reward of $0$. Let $t_s=1$. A stationary random switching policy cannot outperform those individual policies, because they will introduce a risk of getting $R=-2$ from $s_2$ and $s_4$. In fact, any such policy that assigns strictly positive probabilities to each action will perform worse than the individual policies. On the other hand, a policy that outperforms the individual policies by optimally switching between the modes exists and achieves $T$ cumulative reward.
\end{proof}

\begin{figure}[t]
    \centering
    \includegraphics[width=0.85\columnwidth]{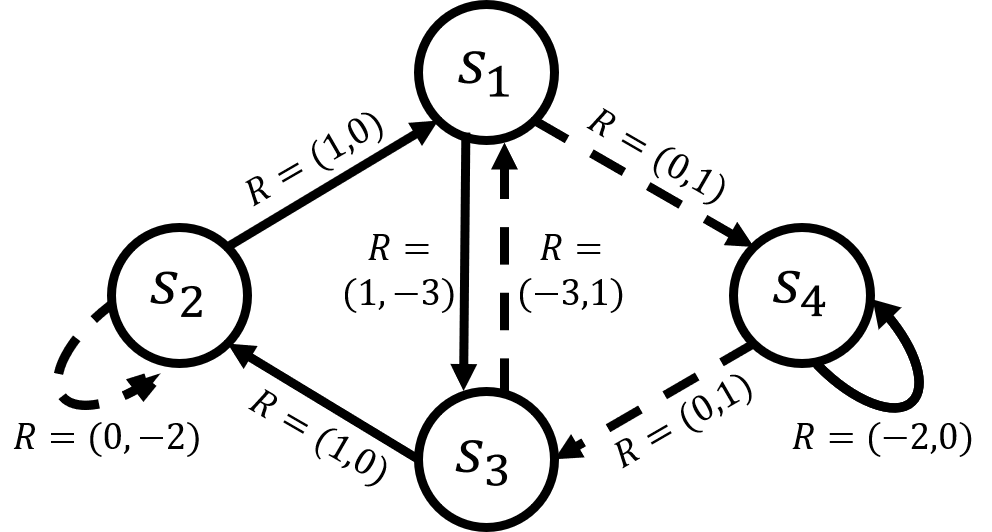}
    \vspace{-10px}
    \caption{While random switching cannot guarantee better performance, an intelligent switching policy outperforms individual low-level policies.}
    \label{fig:remark_mdp}
    \vspace{-15px}
\end{figure}

Unfortunately, the assumptions of Proposition~\ref{prop:random} may not hold for driving in general, and Remark~\ref{remark:mdp} shows that a stationary random switching strategy may perform poorly. Next, we show that the solution to \eqref{eq:rl_optimization} yields a good policy.

\begin{proposition}
The optimal solution to \eqref{eq:rl_optimization} is at least as good as keeping the same low-level policy throughout the episode in terms of the expected cumulative reward.
\end{proposition}
\begin{proof}
Since $\Pi(o^j)=i$ for $\forall o^j\in\Omega$ for any $i$ is a feasible solution to \eqref{eq:rl_optimization}, the optimal solution is guaranteed to be at least as good as keeping the same low-level policy in terms of the objective, i.e. the expected cumulative reward.
\end{proof}

In \textsc{H-ReIL}, we decompose the complicated task of driving in near-accident scenarios into two levels, where the low-level learns basic policies with IL to realize relatively easier goals, and the high-level learns a meta-policy using RL to switch between different low-level policies to maximize the cumulative reward. \textbf{The mode switching can model rapid phase transitions. With the reduced action space and fewer time steps, the high-level RL can explore all the states efficiently to address state coverage.} The two-level architecture makes both IL and RL much easier, and learns a policy to drive efficiently and safely in near-accident scenarios.

\section{Experiments}

A video giving an overview of our experiments, as well as the proposed framework, is at \url{https://youtu.be/CY24zlC_HdI}. Below, we describe our experiment settings.

\begin{figure*}[th]
    \centering
    \includegraphics[width=.92\textwidth]{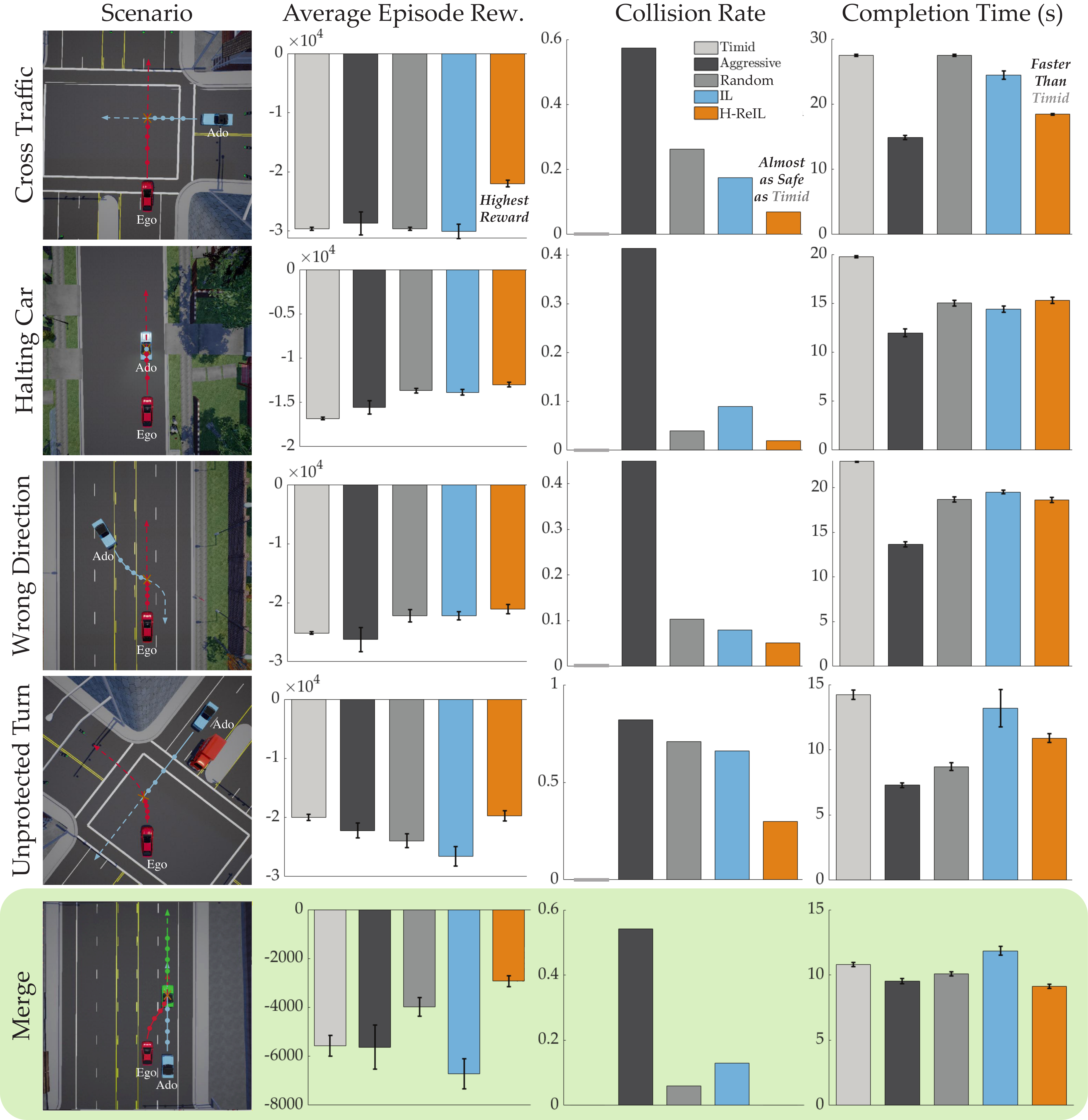}
    \vspace{-8px}
    \caption{The scenario illustration, average episode reward, collision rate, completion time for each scenario and each policy in CARLA simulator. In the scenario visualizations, the ego car is always red and the ado car is blue.}
    \label{fig:bar_plot}
    \vspace{-15px}
\end{figure*}

\subsection{Environment}
We consider the environment where the ego car navigates in the presence of an ado car. The framework extends to cases with multiple environment cars easily. In order to model near-accident scenarios, we let the ado car employ a policy to increase the possibility of collision with the ego car.

\subsection{Scenarios}
We design five near-accident scenarios, each of which is visualized in Fig.~\ref{fig:bar_plot} and described subsequently.

\noindent \textbf{1) Cross Traffic.} The ego car seeks to cross the intersection, but a building occludes the ado car (Fig.~\ref{fig:bar_plot}, row 1).

\noindent \textbf{2) Halting Car.} The ego car drives behind the ado car, which occasionally stops abruptly (Fig.~\ref{fig:bar_plot}, row 2).

\noindent \textbf{3) Wrong Direction.} The ado car, which drives in the opposite direction, cuts into the ego car's lane (Fig.~\ref{fig:bar_plot}, row 3).

\noindent \textbf{4) Unprotected Turn.} The ego car seeks to make a left turn, but a truck occludes the oncoming ado car (Fig.~\ref{fig:bar_plot}, row 4).

\noindent \textbf{5) Merge.} The ego car wants to cut between the ado car and another car in the front, who follows a fixed driving policy. However, the ado car can aggressively accelerate to prevent it from merging (Fig.~\ref{fig:bar_plot}, row 5). 

For each scenario, we have two settings: difficult and easy. The difficult setting is described above where the ado car acts carelessly or aggressively, and is likely to collide with the ego car. The easy setting either completely removes the ado car from the environment or makes it impossible to collide with the ego car. In simulation, we sample between these two settings uniformly at random for each scenario. In addition, we also perturb the initial positions of both cars with some uniform random noise in their nominal directions.

\subsection{Simulators}
\textbf{CARLO}\footnote{Publicly available at \url{https://github.com/Stanford-ILIAD/CARLO}.} is our in-home 2D driving simulator that models the physics and handles the visualizations in a simplistic way (see Fig.~\ref{fig:trajectory}). Assuming point-mass dynamics model as in \cite{sadigh2017active}, CARLO simulates vehicles, buildings and pedestrians.

While CARLO does not provide realistic visualizations other than two-dimensional diagrams, it is useful for developing control models and collecting large amounts of data. Therefore, we use CARLO as a simpler environment where we assume perception is handled, and so we can directly use the noisy measurements of other vehicles' speeds and positions (if not occluded) in addition to the state of the ego vehicle.

\textbf{CARLA} \cite{dosovitskiy2017carla} is an open-source simulator for autonomous driving research, which provides realistic urban environments for training and validation of autonomous driving systems. Specifically, CARLA enables users to create various digital assets (pedestrians, buildings, vehicles) and specifies sensor suites and environmental conditions flexibly. We use CARLA as a more realistic simulator than CARLO.

For both CARLO and CARLA, the control inputs for the vehicles are throttle/brake and steering.

\begin{figure*}[th]
    \centering
    \includegraphics[width=\textwidth]{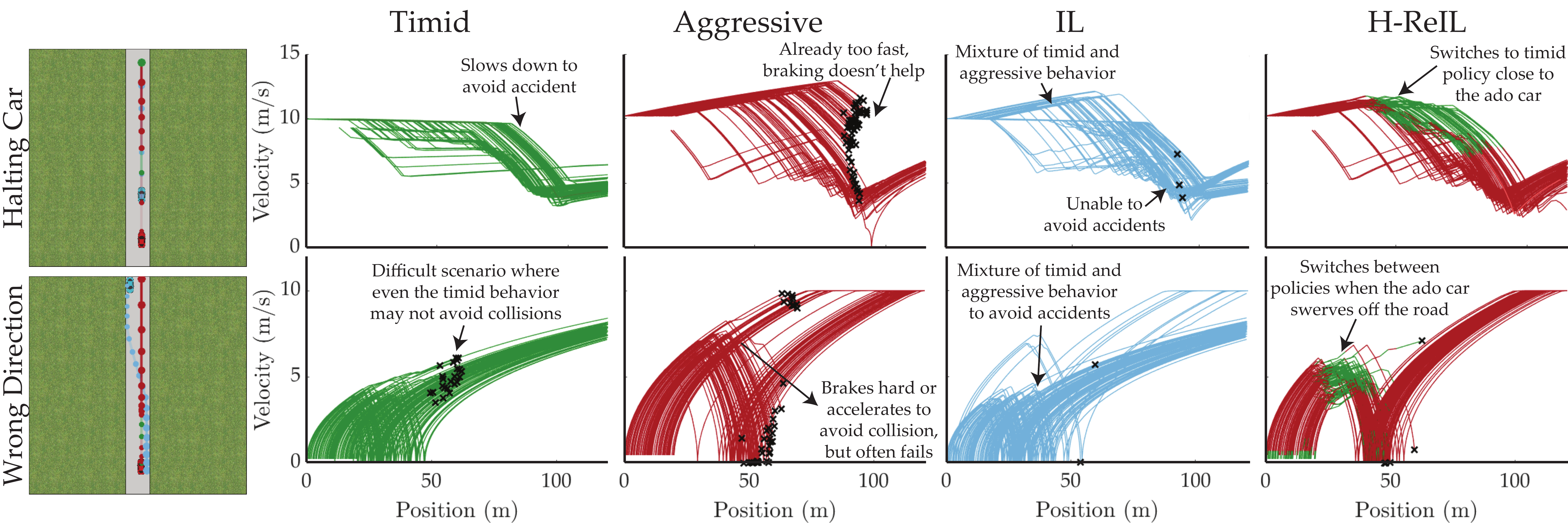}
    \vspace{-21px}
    \caption{The plots of velocity vs position of the ego car under Halting Car and Wrong Direction scenarios with \textsc{Timid}, \textsc{Aggressive}, \textsc{IL} and \textsc{H-ReIL} in CARLO. The green and red colors correspond to the selections of timid and aggressive modes, respectively. The black crosses show collisions where the episode terminates. The episode also terminates when the ego car arrives at the predefined destinations.}
    \label{fig:line_plot}
    \vspace{-18px}
\end{figure*}

\subsection{Modes}
While \textsc{H-ReIL} can be used with any finite number of modes, we consider two in this paper ($n=2$): aggressive and timid modes. In the former, the ego car favors efficiency over safety: It drives fast and frequently collides with the ado car. In the timid mode, the ego car drives in a safe way to avoid all potential accidents: It slows down whenever there is even a slight risk of an accident. The high-level agent learns to switch between the two modes to achieve our final goal: \emph{driving safely and efficiently in near-accident scenarios}.

For the near-accident driving setting, having two modes of driving -- aggressive and timid -- is arguably the most natural and realistic choice. Since humans often do not optimize for other nuanced metrics, such as comfort, in a near-accident scenario and the planning horizon of our high-level controller is extremely short, there is a limited amount of diversity that different modes of driving would provide, which makes having extra modes unrealistic and unnecessary in our setting. 

For our simulations on the first four scenarios (other than Merge), we collect data from the hand-coded aggressive and timid modes for the ego car based on rules around the positions and velocities of the vehicles involved. While both modes try to avoid accidents and reach destinations; their reaction times, acceleration rates and willingness to take risks differ.

For the Merge scenario, we collected real driving data from a driver
who tried to drive either aggressively or timidly. We collected human data only in CARLA due to its more realistic visualizations and dynamics model.

In each of the first four scenarios, we separately collect aggressive and timid driving data as expert demonstrations for the aggressive and timid modes, denoted by $\mathcal{H}_\mathrm{agg}$ and $\mathcal{H}_\mathrm{tim}$, respectively. In CARLO, which enables fast data collection, we collected $80000$ episodes per mode. In CARLA, which includes perception data, we collected $100$ episodes per mode.

\subsection{Compared Methods}
We compare \textsc{H-ReIL} with the following policies:
\begin{enumerate}
    \item \textsc{IL}. $\pi^\mathrm{IL}$ trained on the mixture of aggressive and timid demonstrations $\mathcal{H}_\mathrm{agg}$ and $\mathcal{H}_\mathrm{tim}$.
    \item \textsc{Aggressive}. $\pi^\mathrm{agg}$ trained only on $\mathcal{H}_\mathrm{agg}$ with IL.
    \item \textsc{Timid}. $\pi^\mathrm{tim}$ trained only on $\mathcal{H}_\mathrm{tim}$ with IL.
    \item \textsc{Random}. $\Pi^\mathrm{rand}$ which selects $\pi^\mathrm{agg}$ or $\pi^\mathrm{tim}$ at every high-level time step uniformly at random.
\end{enumerate}

\subsection{Implementation Details}\label{sec:implementation_details}
\noindent\textbf{CARLO.} The observations include ego car location and velocity. They also include the location and the velocity of the ado car, if not occluded, perturbed with Gaussian noise.

These are then fed into a neural network policy with two fully-connected hidden layers to output the high-level decision. The same information are also fed into a neural network with only a single fully-connected hidden layer to obtain features. Depending on the high-level mode selection, these features are then passed into another fully connected neural network with a single hidden layer, which outputs the controls.

\noindent\textbf{CARLA.} The observations consist of ego car location, velocity and a front-view image for the first four scenarios. Merge scenario has additional right-front and right-view images to gain necessary information specific to the scenario. 

For the first four scenarios, we use an object detection model, Faster-RCNN with R-50-FPN backbone \cite{ren2015faster}, to detect the cars in the front-view images and generate a black image with only the bounding boxes colored white, which we call the detection image. It provides information of the ado car more clearly and alleviates the environmental noise. We do not apply this technique to the Merge scenario because the ado car usually drives in parallel with the ego car and its shape is only partially observable in some views. Instead, we use the original RGB images for the Merge scenario.

We then design another network consisting of a convolutional neural network encoder and a fully-connected network encoder. The convolutional encoder encodes the detection image and the fully-connected encoder encodes the location and velocity information (of the ego car) into features.

The high-level RL policy feeds these features into a fully-connected network to output which mode the ego car will follow. We then feed the features to the chosen low-level IL policy composed of fully-connected layers, at the next $t_s$ low-level time steps to obtain the controls. We use Proximal Policy Optimization (PPO) \cite{schulman2017proximal} for the high-level agent of \textsc{H-ReIL}.

For $\textsc{IL}$, we use a network structure similar to our approach but without branching  since there is no mode selection.

\begin{figure*}[th]
    \centering
    \includegraphics[width=.95\textwidth]{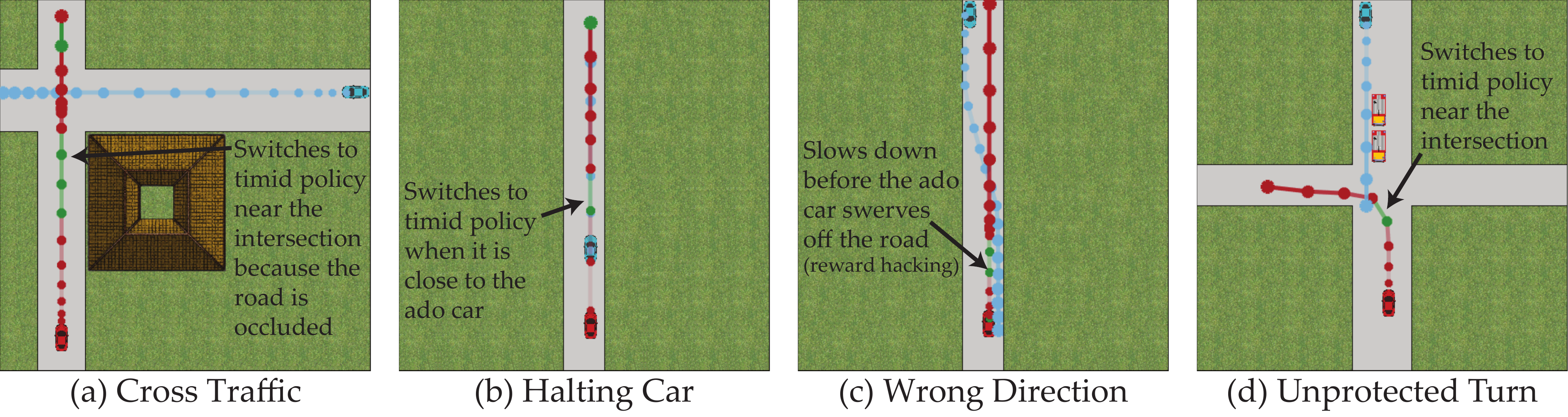}
    \vspace{-10px}
    \caption{Visualization of locations at each time step of the ego car and the ado car in CARLO simulator. The blue color shows trajectory of the ado car. Green means selecting the timid policy while red means selecting the aggressive policy.}
    \label{fig:trajectory}
    \vspace{-15px}
\end{figure*}

\section{Results}
\subsection{Simulations}
We compare the \emph{average episode reward}, \emph{collision rate}, and \emph{completion time} of different methods under all scenarios with both simulators. We compute these metrics for each model and scenario averaged over 100 test runs.

For the simple reward of the high-level agent, we select the trade-off between efficiency (time/distance penalty) and safety (collision penalty) such that the high-level policy cannot na\"ively bias to a single low-level policy. The collision rate is only computed for the episodes with the difficult setting.

As shown in Fig.~\ref{fig:bar_plot} for CARLA, our \textsc{H-ReIL} framework is better than or comparable to other methods in terms of the average episode reward under all five scenarios, which demonstrates the high-level RL agent can effectively learn a smart switching between low-level policies. \textsc{H-ReIL} framework usually outperforms \textsc{IL} with a large margin, supporting the claim that in near-accident scenarios, training a generalizable IL policy requires a lot of demonstrations. \textit{Inadequate demonstrations cause the IL policy to fail in several scenarios}. 

In terms of collision rate and completion time, \textsc{H-ReIL} achieves a collision rate lower than \textsc{IL}, \textsc{Aggressive} and \textsc{Random} while comparable to \textsc{Timid}. \textsc{H-ReIL} also achieves a completion time shorter than \textsc{IL} and \textsc{Timid} while comparable to \textsc{Random}. These demonstrate \textsc{H-ReIL} achieves a good trade-off between efficiency and safety.

\subsection{User Studies}
Having collected real driving data in CARLA for the Merge scenario, we generated a test set that consists of $18$ trajectories for each of \textsc{Aggressive}, \textsc{Timid}, \textsc{IL} and \textsc{H-ReIL}. We then recruited $49$ subjects through Amazon Mechanical Turk to evaluate how good the driving is on a $7$-point Likert scale (1 - least preferred, 7 - most preferred). Figure~\ref{fig:user_study_results} shows the users prefer \textsc{H-ReIL} over the other methods. The differences between \textsc{H-ReIL} and the other methods are statistically significant with $p<0.005$ (two-sample t-test).

\begin{figure}[ht]
    \centering
    \includegraphics[width=\columnwidth]{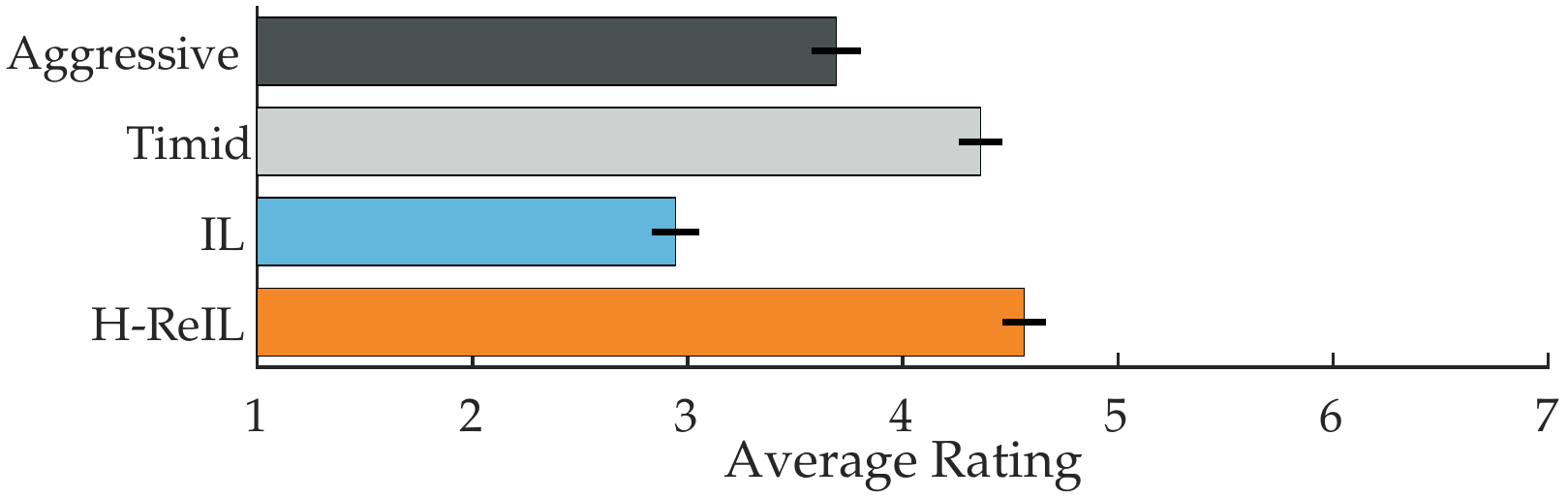}
    \vspace{-20px}
    \caption{User study results are shown. Users rate \textsc{H-ReIL} significantly higher than the other methods ($p<0.005$).}
    \label{fig:user_study_results}
    \vspace{-15px}
\end{figure}

\section{Analysis}
\noindent\textbf{Velocity Analysis.}
We visualize the relation between the velocity and the position of the ego car in its nominal direction in Fig.~\ref{fig:line_plot} for the Halting Car and the Wrong Direction scenarios in CARLO. We selected these two scenarios for visualization as the ego does not change direction.

We observe \textsc{Timid} always drives with a relatively low speed while \textsc{Aggressive} drives fast but collides with the ado car more often. Compared with these two, \textsc{H-ReIL} and \textsc{IL} drive at a medium speed while \textsc{H-ReIL} achieves a relatively higher speed than \textsc{IL} with comparable number of accidents. 

In particular, there is an obvious phase transition in both scenarios (about $[35, 75]$ for the Halting Car and $[25, 45]$ for the Wrong Direction) where a collision is very likely to occur. Baseline models learned by plain IL, cannot model such phase transitions well. Instead, \textsc{H-ReIL} switches the modes to model such phase transitions: it selects the timid mode in the risky states to ensure safety while selecting the aggressive policy in other states to maximize efficiency. This intelligent mode switching enables \textsc{H-ReIL} to drive reasonably under different situations: slowly and cautiously under uncertainty, and fast when there is no potential risk.

\noindent\textbf{Policy Visualization.}
We visualize the locations of the cars in Fig.~\ref{fig:trajectory} in CARLO. We observe that \textsc{H-ReIL} usually chooses the timid policy at the areas that have a collision risk while staying aggressive at other locations when it is safe to do so. These support that our high-level policy makes correct decisions under different situations.

\begin{figure}[ht]
    \centering
    \vspace{-3px}
    \includegraphics[width=0.95\columnwidth]{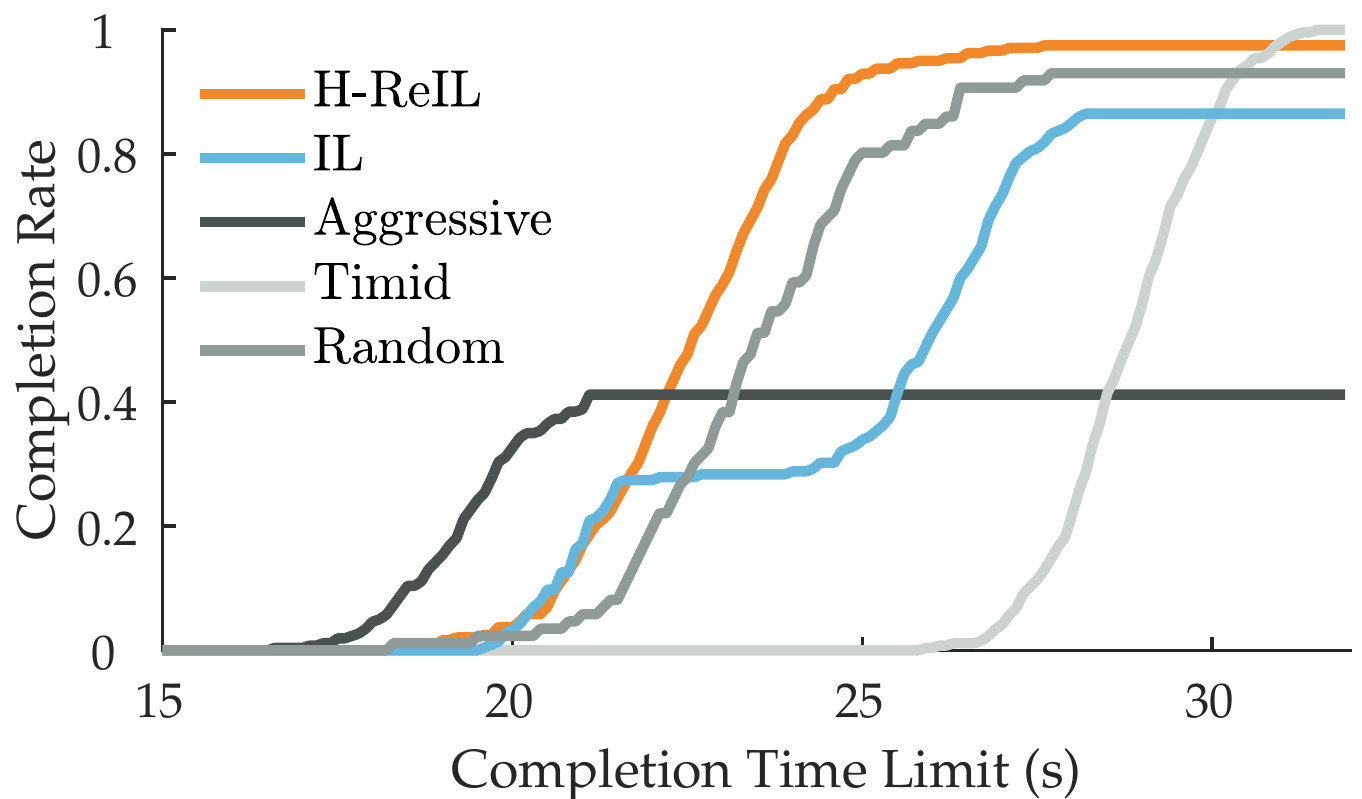}
    \vspace{-8px}
    \caption{The completion rate with varying time limits. The completion rate is the proportion of the trajectories in which the ego car safely reaches the destination within the time limit.}
    \label{fig:frontier}
    \vspace{-7px}
\end{figure}

\noindent\textbf{Completion within Time Limit.} We plot the completion rate with respect to varying time limits for the ego car in Fig.~\ref{fig:frontier} in CARLA for the Cross Traffic scenario. The completion rate is the portion within $500$ runs that the ego car reaches the destination within the time limit. Overall, we observe \textsc{H-ReIL} achieves the best trade-off. While \textsc{Aggressive} achieves higher completion rates for the low time limits, it cannot improve further with the increasing limit with collisions.

We also observe the trajectories of \textsc{IL} are divided into two clusters. The group that achieves lower time limit (20-22s) imitates the aggressive policy more but has lower completion rate. The other group that corresponds to the higher time limit (25-28s) imitates the timid policy more but has better completion rate. This demonstrates \textsc{IL} directly imitates the two modes and learns a mild aggressive or a mild timid policy while it does not learn when to use each mode. On the other hand, \textsc{H-ReIL} consistently achieves higher or comparable completion rate than \textsc{IL} and \textsc{Random}, showing that our high-level RL agent can learn when to switch between the modes to safely arrive at the destination efficiently.

\section{Conclusion}
\noindent\textbf{Summary.} In this work, we proposed a novel hierarchy with reinforcement learning and imitation learning to achieve safe and efficient driving in near-accident scenarios. By learning low-level policies using IL from drivers with different characteristics, such as different aggressiveness levels, and training a high-level RL policy that makes the decision of which low-level policy to use, our method \textsc{H-ReIL} achieves a good trade-off between safety and efficiency. Simulations and user studies show it is preferred over the compared methods.

\noindent\textbf{Limitations and Future Work.} Although \textsc{H-ReIL} is generalizable to any finite number of modes, we only considered $n\!=\!2$. Having more than $2$ modes, for which our preliminary experiments have given positive results, can be useful for other robotic tasks. Also, we hand-designed the near-accident scenarios in this work. Generating them automatically as in \cite{o2018scalable} could enable broader evaluation in realistic scenarios. 

\section*{Acknowledgments}
The authors thank Derek Phillips for the help with CARLA simulator, Wentao Zhong and Jiaqiao Zhang for additional experiments with \textsc{H-ReIL}, and acknowledge funding by FLI grant RFP2-000. Toyota Research Institute (``TRI") provided funds to assist the authors with their research but this article solely reflects the opinions and conclusions of its authors and not TRI or any other Toyota entity.

\bibliographystyle{unsrtnat}
\small\bibliography{refs}

\begin{thebibliography}{55}
\providecommand{\natexlab}[1]{#1}
\providecommand{\url}[1]{\texttt{#1}}
\expandafter\ifx\csname urlstyle\endcsname\relax
  \providecommand{\doi}[1]{doi: #1}\else
  \providecommand{\doi}{doi: \begingroup \urlstyle{rm}\Url}\fi

\bibitem[Pomerleau(1989)]{pomerleau1989alvinn}
Dean~A Pomerleau.
\newblock Alvinn: An autonomous land vehicle in a neural network.
\newblock In \emph{Advances in neural information processing systems}, pages
  305--313, 1989.

\bibitem[Bojarski et~al.(2016)Bojarski, Del~Testa, Dworakowski, Firner, Flepp,
  Goyal, Jackel, Monfort, Muller, Zhang, et~al.]{bojarski2016end}
Mariusz Bojarski, Davide Del~Testa, Daniel Dworakowski, Bernhard Firner, Beat
  Flepp, Prasoon Goyal, Lawrence~D Jackel, Mathew Monfort, Urs Muller, Jiakai
  Zhang, et~al.
\newblock End to end learning for self-driving cars.
\newblock \emph{arXiv preprint arXiv:1604.07316}, 2016.

\bibitem[Amini et~al.(2019)Amini, Rosman, Karaman, and
  Rus]{amini2019variational}
Alexander Amini, Guy Rosman, Sertac Karaman, and Daniela Rus.
\newblock Variational end-to-end navigation and localization.
\newblock In \emph{2019 International Conference on Robotics and Automation
  (ICRA)}, pages 8958--8964. IEEE, 2019.

\bibitem[Sadigh et~al.(2016{\natexlab{a}})Sadigh, Sastry, Seshia, and
  Dragan]{sadigh2016planning}
Dorsa Sadigh, S.~Shankar Sastry, Sanjit~A. Seshia, and Anca~D. Dragan.
\newblock Planning for autonomous cars that leverage effects on human actions.
\newblock In \emph{Proceedings of Robotics: Science and Systems (RSS)}, June
  2016{\natexlab{a}}.
\newblock \doi{10.15607/RSS.2016.XII.029}.

\bibitem[Sadigh et~al.(2016{\natexlab{b}})Sadigh, Sastry, Seshia, and
  Dragan]{sadigh2016information}
Dorsa Sadigh, S.~Shankar Sastry, Sanjit~A. Seshia, and Anca Dragan.
\newblock Information gathering actions over human internal state.
\newblock In \emph{Proceedings of the {IEEE}, /{RSJ}, International Conference
  on Intelligent Robots and Systems (IROS)}, pages 66--73. IEEE, October
  2016{\natexlab{b}}.
\newblock \doi{10.1109/IROS.2016.7759036}.

\bibitem[Sadigh et~al.(2018)Sadigh, Landolfi, Sastry, Seshia, and
  Dragan]{sadigh2018planning}
Dorsa Sadigh, Nick Landolfi, S.~Shankar Sastry, Sanjit~A. Seshia, and Anca~D.
  Dragan.
\newblock Planning for cars that coordinate with people: Leveraging effects on
  human actions for planning and active information gathering over human
  internal state.
\newblock \emph{Autonomous Robots (AURO)}, 42\penalty0 (7):\penalty0
  1405--1426, October 2018.
\newblock ISSN 1573-7527.
\newblock \doi{10.1007/s10514-018-9746-1}.

\bibitem[Lee et~al.(2017)Lee, Choi, Vernaza, Choy, Torr, and
  Chandraker]{lee2017desire}
Namhoon Lee, Wongun Choi, Paul Vernaza, Christopher~B Choy, Philip~HS Torr, and
  Manmohan Chandraker.
\newblock Desire: Distant future prediction in dynamic scenes with interacting
  agents.
\newblock In \emph{Proceedings of the IEEE Conference on Computer Vision and
  Pattern Recognition}, pages 336--345, 2017.

\bibitem[Biyik and Sadigh(2018)]{biyik2018batch}
Erdem Biyik and Dorsa Sadigh.
\newblock Batch active preference-based learning of reward functions.
\newblock In \emph{Proceedings of the 2nd Conference on Robot Learning (CoRL)},
  volume~87 of \emph{Proceedings of Machine Learning Research}, pages 519--528.
  PMLR, October 2018.

\bibitem[Codevilla et~al.(2019)Codevilla, Santana, L{\'o}pez, and
  Gaidon]{codevilla2019exploring}
Felipe Codevilla, Eder Santana, Antonio~M L{\'o}pez, and Adrien Gaidon.
\newblock Exploring the limitations of behavior cloning for autonomous driving.
\newblock In \emph{Proceedings of the IEEE International Conference on Computer
  Vision}, pages 9329--9338, 2019.

\bibitem[Kwon et~al.(2020)Kwon, Biyik, Talati, Bhasin, Losey, and
  Sadigh]{kwon2020when}
Minae Kwon, Erdem Biyik, Aditi Talati, Karan Bhasin, Dylan~P. Losey, and Dorsa
  Sadigh.
\newblock When humans aren't optimal: Robots that collaborate with risk-aware
  humans.
\newblock In \emph{ACM/IEEE International Conference on Human-Robot Interaction
  (HRI)}, March 2020.
\newblock \doi{10.1145/3319502.3374832}.

\bibitem[Basu et~al.(2019)Basu, Biyik, He, Singhal, and Sadigh]{basu2019active}
Chandrayee Basu, Erdem Biyik, Zhixun He, Mukesh Singhal, and Dorsa Sadigh.
\newblock Active learning of reward dynamics from hierarchical queries.
\newblock In \emph{Proceedings of the IEEE/RSJ International Conference on
  Intelligent Robots and Systems (IROS)}, November 2019.
\newblock \doi{10.1109/IROS40897.2019.8968522}.

\bibitem[Wulfmeier et~al.(2017)Wulfmeier, Rao, Wang, Ondruska, and
  Posner]{wulfmeier2017large}
Markus Wulfmeier, Dushyant Rao, Dominic~Zeng Wang, Peter Ondruska, and Ingmar
  Posner.
\newblock Large-scale cost function learning for path planning using deep
  inverse reinforcement learning.
\newblock \emph{The International Journal of Robotics Research}, 36\penalty0
  (10):\penalty0 1073--1087, 2017.

\bibitem[Makantasis et~al.(2019)Makantasis, Kontorinaki, and
  Nikolos]{makantasis2019deep}
Konstantinos Makantasis, Maria Kontorinaki, and Ioannis Nikolos.
\newblock A deep reinforcement learning driving policy for autonomous road
  vehicles.
\newblock \emph{arXiv preprint arXiv:1905.09046}, 2019.

\bibitem[Sallab et~al.(2017)Sallab, Abdou, Perot, and Yogamani]{sallab2017deep}
Ahmad~EL Sallab, Mohammed Abdou, Etienne Perot, and Senthil Yogamani.
\newblock Deep reinforcement learning framework for autonomous driving.
\newblock \emph{Electronic Imaging}, 2017\penalty0 (19):\penalty0 70--76, 2017.

\bibitem[Codevilla et~al.(2018)Codevilla, Miiller, L{\'o}pez, Koltun, and
  Dosovitskiy]{codevilla2018end}
Felipe Codevilla, Matthias Miiller, Antonio L{\'o}pez, Vladlen Koltun, and
  Alexey Dosovitskiy.
\newblock End-to-end driving via conditional imitation learning.
\newblock In \emph{2018 IEEE International Conference on Robotics and
  Automation (ICRA)}, pages 1--9. IEEE, 2018.

\bibitem[Zhang and Cho(2017)]{zhang2017query}
Jiakai Zhang and Kyunghyun Cho.
\newblock Query-efficient imitation learning for end-to-end simulated driving.
\newblock In \emph{Thirty-First AAAI Conference on Artificial Intelligence},
  2017.

\bibitem[Bansal et~al.(2018)Bansal, Krizhevsky, and
  Ogale]{bansal2018chauffeurnet}
Mayank Bansal, Alex Krizhevsky, and Abhijit Ogale.
\newblock Chauffeurnet: Learning to drive by imitating the best and
  synthesizing the worst.
\newblock \emph{arXiv preprint arXiv:1812.03079}, 2018.

\bibitem[Chen et~al.(2015)Chen, Seff, Kornhauser, and
  Xiao]{chen2015deepdriving}
Chenyi Chen, Ari Seff, Alain Kornhauser, and Jianxiong Xiao.
\newblock Deepdriving: Learning affordance for direct perception in autonomous
  driving.
\newblock In \emph{Proceedings of the IEEE International Conference on Computer
  Vision}, pages 2722--2730, 2015.

\bibitem[Bojarski et~al.(2017)Bojarski, Yeres, Choromanska, Choromanski,
  Firner, Jackel, and Muller]{bojarski2017explaining}
Mariusz Bojarski, Philip Yeres, Anna Choromanska, Krzysztof Choromanski,
  Bernhard Firner, Lawrence Jackel, and Urs Muller.
\newblock Explaining how a deep neural network trained with end-to-end learning
  steers a car.
\newblock \emph{arXiv preprint arXiv:1704.07911}, 2017.

\bibitem[Kuefler et~al.(2017)Kuefler, Morton, Wheeler, and
  Kochenderfer]{kuefler2017imitating}
Alex Kuefler, Jeremy Morton, Tim Wheeler, and Mykel Kochenderfer.
\newblock Imitating driver behavior with generative adversarial networks.
\newblock In \emph{2017 IEEE Intelligent Vehicles Symposium (IV)}, pages
  204--211. IEEE, 2017.

\bibitem[Pan et~al.(2017)Pan, You, Wang, and Lu]{pan2017virtual}
Xinlei Pan, Yurong You, Ziyan Wang, and Cewu Lu.
\newblock Virtual to real reinforcement learning for autonomous driving.
\newblock \emph{arXiv preprint arXiv:1704.03952}, 2017.

\bibitem[Huang et~al.(2019)Huang, McGill, Williams, Fletcher, and
  Rosman]{huang2019uncertainty}
Xin Huang, Stephen~G McGill, Brian~C Williams, Luke Fletcher, and Guy Rosman.
\newblock Uncertainty-aware driver trajectory prediction at urban
  intersections.
\newblock In \emph{2019 International Conference on Robotics and Automation
  (ICRA)}, pages 9718--9724. IEEE, 2019.

\bibitem[M{\"u}ller et~al.(2018)M{\"u}ller, Dosovitskiy, Ghanem, and
  Koltun]{muller2018driving}
Matthias M{\"u}ller, Alexey Dosovitskiy, Bernard Ghanem, and Vladlen Koltun.
\newblock Driving policy transfer via modularity and abstraction.
\newblock \emph{arXiv preprint arXiv:1804.09364}, 2018.

\bibitem[Paden et~al.(2016)Paden, {\v{C}}{\'a}p, Yong, Yershov, and
  Frazzoli]{paden2016survey}
Brian Paden, Michal {\v{C}}{\'a}p, Sze~Zheng Yong, Dmitry Yershov, and Emilio
  Frazzoli.
\newblock A survey of motion planning and control techniques for self-driving
  urban vehicles.
\newblock \emph{IEEE Transactions on intelligent vehicles}, 1\penalty0
  (1):\penalty0 33--55, 2016.

\bibitem[Schwarting et~al.(2018)Schwarting, Alonso-Mora, and
  Rus]{schwarting2018planning}
Wilko Schwarting, Javier Alonso-Mora, and Daniela Rus.
\newblock Planning and decision-making for autonomous vehicles.
\newblock \emph{Annual Review of Control, Robotics, and Autonomous Systems},
  2018.

\bibitem[Urmson et~al.(2008)Urmson, Anhalt, Bagnell, Baker, Bittner, Clark,
  Dolan, Duggins, Galatali, Geyer, et~al.]{urmson2008autonomous}
Chris Urmson, Joshua Anhalt, Drew Bagnell, Christopher Baker, Robert Bittner,
  MN~Clark, John Dolan, Dave Duggins, Tugrul Galatali, Chris Geyer, et~al.
\newblock Autonomous driving in urban environments: Boss and the urban
  challenge.
\newblock \emph{Journal of Field Robotics}, 25\penalty0 (8):\penalty0 425--466,
  2008.

\bibitem[Montemerlo et~al.(2008)Montemerlo, Becker, Bhat, Dahlkamp, Dolgov,
  Ettinger, Haehnel, Hilden, Hoffmann, Huhnke, et~al.]{montemerlo2008junior}
Michael Montemerlo, Jan Becker, Suhrid Bhat, Hendrik Dahlkamp, Dmitri Dolgov,
  Scott Ettinger, Dirk Haehnel, Tim Hilden, Gabe Hoffmann, Burkhard Huhnke,
  et~al.
\newblock Junior: The stanford entry in the urban challenge.
\newblock \emph{Journal of field Robotics}, 25\penalty0 (9):\penalty0 569--597,
  2008.

\bibitem[Muller et~al.(2006)Muller, Ben, Cosatto, Flepp, and
  Cun]{muller2006off}
Urs Muller, Jan Ben, Eric Cosatto, Beat Flepp, and Yann~L Cun.
\newblock Off-road obstacle avoidance through end-to-end learning.
\newblock In \emph{Advances in neural information processing systems}, pages
  739--746, 2006.

\bibitem[Ross et~al.(2011)Ross, Gordon, and Bagnell]{ross2011no}
St{\'e}phane Ross, Geoffrey~J Gordon, and J~Andrew Bagnell.
\newblock No-regret reductions for imitation learning and structured
  prediction.
\newblock In \emph{In AISTATS}. Citeseer, 2011.

\bibitem[Ho and Ermon(2016)]{ho2016generative}
Jonathan Ho and Stefano Ermon.
\newblock Generative adversarial imitation learning.
\newblock In \emph{Advances in neural information processing systems}, pages
  4565--4573, 2016.

\bibitem[Song et~al.(2018)Song, Ren, Sadigh, and Ermon]{song2018multi}
Jiaming Song, Hongyu Ren, Dorsa Sadigh, and Stefano Ermon.
\newblock Multi-agent generative adversarial imitation learning.
\newblock In \emph{Advances in Neural Information Processing Systems (NIPS)},
  pages 7461--7472. Curran Associates, Inc., December 2018.

\bibitem[Abbeel and Ng(2004)]{abbeel2004apprenticeship}
Pieter Abbeel and Andrew~Y Ng.
\newblock Apprenticeship learning via inverse reinforcement learning.
\newblock In \emph{Proceedings of the twenty-first international conference on
  Machine learning}, page~1. ACM, 2004.

\bibitem[Ziebart et~al.(2008)Ziebart, Maas, Bagnell, and
  Dey]{ziebart2008maximum}
Brian~D Ziebart, Andrew~L Maas, J~Andrew Bagnell, and Anind~K Dey.
\newblock Maximum entropy inverse reinforcement learning.
\newblock In \emph{Aaai}, volume~8, pages 1433--1438. Chicago, IL, USA, 2008.

\bibitem[Levine and Koltun(2012)]{levine2012continuous}
Sergey Levine and Vladlen Koltun.
\newblock Continuous inverse optimal control with locally optimal examples.
\newblock \emph{arXiv preprint arXiv:1206.4617}, 2012.

\bibitem[Finn et~al.(2016)Finn, Levine, and Abbeel]{finn2016guided}
Chelsea Finn, Sergey Levine, and Pieter Abbeel.
\newblock Guided cost learning: Deep inverse optimal control via policy
  optimization.
\newblock In \emph{International Conference on Machine Learning}, pages 49--58,
  2016.

\bibitem[Chen et~al.(2019)Chen, Yuan, and Tomizuka]{chen2019model}
Jianyu Chen, Bodi Yuan, and Masayoshi Tomizuka.
\newblock Model-free deep reinforcement learning for urban autonomous driving.
\newblock \emph{arXiv preprint arXiv:1904.09503}, 2019.

\bibitem[Youssef and Houda(2019)]{youssef2019deep}
Fenjiro Youssef and Benbrahim Houda.
\newblock Deep reinforcement learning with external control: self-driving car
  application.
\newblock In \emph{Proceedings of the 4th International Conference on Smart
  City Applications}, page~58. ACM, 2019.

\bibitem[Shalev-Shwartz et~al.(2016)Shalev-Shwartz, Shammah, and
  Shashua]{shalev2016safe}
Shai Shalev-Shwartz, Shaked Shammah, and Amnon Shashua.
\newblock Safe, multi-agent, reinforcement learning for autonomous driving.
\newblock \emph{arXiv preprint arXiv:1610.03295}, 2016.

\bibitem[Tram et~al.(2019)Tram, Batkovic, Ali, and
  Sj{\"o}berg]{tram2019learning}
Tommy Tram, Ivo Batkovic, Mohammad Ali, and Jonas Sj{\"o}berg.
\newblock Learning when to drive in intersections by combining reinforcement
  learning and model predictive control.
\newblock In \emph{2019 IEEE Intelligent Transportation Systems Conference
  (ITSC)}, pages 3263--3268. IEEE, 2019.

\bibitem[Gupta et~al.(2019)Gupta, Kumar, Lynch, Levine, and
  Hausman]{gupta2019relay}
Abhishek Gupta, Vikash Kumar, Corey Lynch, Sergey Levine, and Karol Hausman.
\newblock Relay policy learning: Solving long-horizon tasks via imitation and
  reinforcement learning.
\newblock In \emph{Proceedings of the 3rd Conference on Robot Learning (CoRL)},
  October 2019.

\bibitem[Dayan and Hinton(1993)]{dayan1993feudal}
Peter Dayan and Geoffrey~E Hinton.
\newblock Feudal reinforcement learning.
\newblock In \emph{Advances in neural information processing systems}, pages
  271--278, 1993.

\bibitem[Kulkarni et~al.(2016)Kulkarni, Narasimhan, Saeedi, and
  Tenenbaum]{kulkarni2016hierarchical}
Tejas~D Kulkarni, Karthik Narasimhan, Ardavan Saeedi, and Josh Tenenbaum.
\newblock Hierarchical deep reinforcement learning: Integrating temporal
  abstraction and intrinsic motivation.
\newblock In \emph{Advances in neural information processing systems}, pages
  3675--3683, 2016.

\bibitem[Vezhnevets et~al.(2017)Vezhnevets, Osindero, Schaul, Heess, Jaderberg,
  Silver, and Kavukcuoglu]{vezhnevets2017feudal}
Alexander~Sasha Vezhnevets, Simon Osindero, Tom Schaul, Nicolas Heess, Max
  Jaderberg, David Silver, and Koray Kavukcuoglu.
\newblock Feudal networks for hierarchical reinforcement learning.
\newblock In \emph{Proceedings of the 34th International Conference on Machine
  Learning-Volume 70}, pages 3540--3549. JMLR. org, 2017.

\bibitem[Stulp and Schaal(2011)]{stulp2011hierarchical}
Freek Stulp and Stefan Schaal.
\newblock Hierarchical reinforcement learning with movement primitives.
\newblock In \emph{2011 11th IEEE-RAS International Conference on Humanoid
  Robots}, pages 231--238. IEEE, 2011.

\bibitem[Strudel et~al.(2019)Strudel, Pashevich, Kalevatykh, Laptev, Sivic, and
  Schmid]{strudel2019combining}
Robin Strudel, Alexander Pashevich, Igor Kalevatykh, Ivan Laptev, Josef Sivic,
  and Cordelia Schmid.
\newblock Combining learned skills and reinforcement learning for robotic
  manipulations.
\newblock \emph{arXiv preprint arXiv:1908.00722}, 2019.

\bibitem[Wu et~al.(2020)Wu, Gupta, and Kochenderfer]{wu2020model}
Bohan Wu, Jayesh~K Gupta, and Mykel Kochenderfer.
\newblock Model primitives for hierarchical lifelong reinforcement learning.
\newblock \emph{Autonomous Agents and Multi-Agent Systems}, 34\penalty0
  (1):\penalty0 1--38, 2020.

\bibitem[Le et~al.(2018)Le, Jiang, Agarwal, Dud{\'\i}k, Yue, and
  Daum{\'e}~III]{le2018hierarchical}
Hoang~M Le, Nan Jiang, Alekh Agarwal, Miroslav Dud{\'\i}k, Yisong Yue, and Hal
  Daum{\'e}~III.
\newblock Hierarchical imitation and reinforcement learning.
\newblock \emph{arXiv preprint arXiv:1803.00590}, 2018.

\bibitem[Qureshi et~al.(2020)Qureshi, Johnson, Qin, Henderson, Boots, and
  Yip]{qureshi2020composing}
Ahmed~H. Qureshi, Jacob~J. Johnson, Yuzhe Qin, Taylor Henderson, Byron Boots,
  and Michael~C. Yip.
\newblock Composing task-agnostic policies with deep reinforcement learning.
\newblock In \emph{International Conference on Learning Representations}, 2020.
\newblock URL \url{https://openreview.net/forum?id=H1ezFREtwH}.

\bibitem[Nair et~al.(2018)Nair, McGrew, Andrychowicz, Zaremba, and
  Abbeel]{nair2018overcoming}
Ashvin Nair, Bob McGrew, Marcin Andrychowicz, Wojciech Zaremba, and Pieter
  Abbeel.
\newblock Overcoming exploration in reinforcement learning with demonstrations.
\newblock In \emph{2018 IEEE International Conference on Robotics and
  Automation (ICRA)}, pages 6292--6299. IEEE, 2018.

\bibitem[Comanici and Precup(2010)]{comanici2010optimal}
Gheorghe Comanici and Doina Precup.
\newblock Optimal policy switching algorithms for reinforcement learning.
\newblock In \emph{Proceedings of the 9th International Conference on
  Autonomous Agents and Multiagent Systems: Volume 1 - Volume 1}, AAMAS ’10,
  page 709–714, Richland, SC, 2010. International Foundation for Autonomous
  Agents and Multiagent Systems.
\newblock ISBN 9780982657119.

\bibitem[Schulman et~al.(2017)Schulman, Wolski, Dhariwal, Radford, and
  Klimov]{schulman2017proximal}
John Schulman, Filip Wolski, Prafulla Dhariwal, Alec Radford, and Oleg Klimov.
\newblock Proximal policy optimization algorithms.
\newblock \emph{arXiv preprint arXiv:1707.06347}, 2017.

\bibitem[Sadigh et~al.(2017)Sadigh, Dragan, Sastry, and
  Seshia]{sadigh2017active}
Dorsa Sadigh, Anca~D. Dragan, S.~Shankar Sastry, and Sanjit~A. Seshia.
\newblock Active preference-based learning of reward functions.
\newblock In \emph{Proceedings of Robotics: Science and Systems (RSS)}, July
  2017.
\newblock \doi{10.15607/RSS.2017.XIII.053}.

\bibitem[Dosovitskiy et~al.(2017)Dosovitskiy, Ros, Codevilla, Lopez, and
  Koltun]{dosovitskiy2017carla}
Alexey Dosovitskiy, German Ros, Felipe Codevilla, Antonio Lopez, and Vladlen
  Koltun.
\newblock Carla: An open urban driving simulator.
\newblock In \emph{Proceedings of the 1st Conference on Robot Learning (CoRL)},
  November 2017.

\bibitem[Ren et~al.(2015)Ren, He, Girshick, and Sun]{ren2015faster}
Shaoqing Ren, Kaiming He, Ross Girshick, and Jian Sun.
\newblock Faster r-cnn: Towards real-time object detection with region proposal
  networks.
\newblock In \emph{Advances in neural information processing systems}, pages
  91--99, 2015.

\bibitem[O'Kelly et~al.(2018)O'Kelly, Sinha, Namkoong, Tedrake, and
  Duchi]{o2018scalable}
Matthew O'Kelly, Aman Sinha, Hongseok Namkoong, Russ Tedrake, and John~C Duchi.
\newblock Scalable end-to-end autonomous vehicle testing via rare-event
  simulation.
\newblock In \emph{Advances in Neural Information Processing Systems}, pages
  9827--9838, 2018.

\end{thebibliography}

\end{document}